\pdfoutput=1 
\documentclass[
]{ceurart}

\pdfoutput=1 
\usepackage{listings}

\usepackage{graphicx}
\usepackage{appendix}
\usepackage{caption}
\usepackage{subcaption}
\usepackage{natbib}
\usepackage{amsmath}
\usepackage{amsfonts} 
\usepackage{amssymb}  
\usepackage{mathtools}
\usepackage{amsthm}
\usepackage{hyperref}
\theoremstyle{plain}
\newtheorem{theorem}{Theorem}[section]
\newtheorem{proposition}[theorem]{Proposition}

\newtheorem{corollary}[theorem]{Corollary}
\theoremstyle{definition}
\newtheorem{definition}[theorem]{Definition}

\theoremstyle{remark}

\begin{document}

\copyrightyear{2024}
\copyrightclause{Copyright for this paper by its authors.
  Use permitted under Creative Commons License Attribution 4.0
  International (CC BY 4.0).}

\conference{Under Review at LNSAI 2024, First International Workshop on Logical Foundations of Neuro-Symbolic AI, IJCAI 2024}

\title{Investigating Symbolic Capabilities of \\ Large Language Models}


\author[1]{Neisarg Dave}[%
email=nud83@psu.edu,
]
\cormark[1]
\address[1]{College of Information Sciences and Technology, The Pennsylvania State University, USA}
\address[2]{Department of Computer Science and Engineering, The Pennsylvania State University, USA}
\address[3]{Department of Computer Science and Engineering, University of South Florida, USA}

\author[2]{Daniel Kifer}[%
]

\author[1]{C. Lee Giles}[%
]

\author[3]{Ankur Mali}[%
]

\cortext[1]{Corresponding author.}

\begin{abstract}
Prompting techniques have significantly enhanced the capabilities of Large Language Models (LLMs) across various complex tasks, including reasoning, planning, and solving math word problems. However, most research has predominantly focused on language-based reasoning and word problems, often overlooking the potential of LLMs in handling symbol-based calculations and reasoning. This study aims to bridge this gap by rigorously evaluating LLMs on a series of symbolic tasks, such as addition, multiplication, modulus arithmetic, numerical precision, and symbolic counting. Our analysis encompasses eight LLMs, including four enterprise-grade and four open-source models, of which three have been pre-trained on mathematical tasks. The assessment framework is anchored in Chomsky’s Hierarchy, providing a robust measure of the computational abilities of these models. The evaluation employs minimally explained prompts alongside the zero-shot Chain of Thoughts technique, allowing models to navigate the solution process autonomously. The findings reveal a significant decline in LLMs' performance on context-free and context-sensitive symbolic tasks as the complexity, represented by the number of symbols, increases. Notably, even the fine-tuned GPT3.5 exhibits only marginal improvements, mirroring the performance trends observed in other models. Across the board, all models demonstrated a limited generalization ability on these symbol-intensive tasks. This research underscores LLMs' challenges with increasing symbolic complexity and highlights the need for specialized training, memory and architectural adjustments to enhance their proficiency in symbol-based reasoning tasks.
\end{abstract}



\maketitle
\section{Introduction}

Large language models are driving revolutionary advancements in artificial intelligence. They are transforming how we interact with and interpret data across various modalities. These models have surpassed human performance in generating text, images, and videos. LLM-augmented systems are revolutionizing traditional expert systems such as search, retrieval, theorem proving, symbolic reasoning, programming, and drug discovery. Amid these advancements, exploring the symbolic capabilities of LLMs, particularly in logical reasoning, mathematical computation, and formal language manipulation, remains a crucial area of investigation.

A key aspect of leveraging LLM capabilities is developing effective prompting techniques. Prompts are carefully designed inputs that guide the model's behavior, enabling it to perform specific tasks or exhibit desired behaviors. These techniques include few-shot and zero-shot learning, where models are prompted with examples or instructions to help them generalize to new tasks without extensive retraining. Other techniques include Chain-of-Thought \cite{wei2022chain}, Tree-of-Thought \cite{yao2024tree}, Graph-of-Thought \cite{besta2024graph} etc. While LLMs have gained prominence, other neural network models have also been studied for their symbol manipulation abilities. This research can be broadly categorized into two areas: directly learning to solve logical and mathematical problems, and research conducted from the perspective of formal methods and automata theory. Recurrent Neural Networks (RNNs) are highly efficient at learning regular grammars, with Second Order RNNs particularly adept at encoding stable states. Memory-augmented models, such as stack-RNN, tape-RNN, and Neural Turing Machines (NTMs), have demonstrated their ability to handle more complex languages. Additionally, transformers, with their attention mechanisms, have been explored for their potential in managing formal languages and symbolic computations. However, RNN-based models often fail to generalize on tasks such as math word problems and logical reasoning.


In this work, we address the symbol manipulation capabilities of Large Language Models (LLMs). We investigate fundamental mathematical operations such as addition, multiplication, and counting. The inherent complexity of these tasks is derived from Chomsky's Hierarchy. Specifically, we present the following research questions:

\noindent
\textbf{RQ1:} Can LLMs apply simple symbolic rules, such as addition and multiplication, stably and repeatedly?

\noindent
\textbf{RQ2:} How precisely can LLMs preserve the construction of a group of symbols and respect the order of operations?

In this study, we first discuss task complexity and the computational requirements necessary to solve these tasks. We then estimate the number of bits required to encode such a machine and compare it to the knowledge tuple encoding in LLMs as described by Deletang et al. \cite{deletang2022neural}. Additionally, we estimate the number of parameters needed for a neural network to solve the given tasks using both methods. Finally, we provide experimental evidence of the performance of LLMs. We evaluate four enterprise LLMs and four open-source LLMs on three essential skills: addition, multiplication, and counting, and design $5$ tasks to answer $RQ1$ and $RQ2$. Utilizing Hindu-Arabic numerals and both uppercase and lowercase English alphabets as symbols, we construct multi-digit numbers from single digits. This approach allows us to assess the precision of the LLMs by increasing the number of digits in the numbers, both before and after the decimal point.

\section{Related Work}

%


Early research demonstrated that recurrent neural networks (RNNs) could learn and encode deterministic finite state automata, enabling them to perform symbolic tasks such as rule-based reasoning and grammatical inference. Omlin and Giles \cite{omlin1996extraction, giles1993extraction, omlin1996constructing} demonstrated the extraction, insertion, and refinement of deterministic finite automata in recurrent neural networks. They also explored higher-order recurrent networks and their applications in grammatical inference \cite{omlin_2nd_order}. However, RNNs have been shown to struggle with tasks extending beyond regular grammars. For instance, RNNs require auxiliary data structures like stacks or tapes to process inputs governed by context-free or context-sensitive grammars. Mali et al. \cite{mali2023computational, mali2021recognizing} and 
Stogin et al. \cite{Stogin_mali} advanced this work by providing theoretical bounds for the expressiveness of RNNs and memory-augmented RNN models. Dave et al. \cite{dave2024stability} empirically showed that second-order connections encode states more stably for formal languages in comparison to first-order RNNs like LSTM and GRU. Symbolic tasks encompass a range of applications, including the manipulation of formal languages and the execution of mathematical operations \cite{mali2021recognizing_a}. Mathematics provides a structured framework for symbolic evaluation, aligning well with the principles of neurosymbolic AI, where symbol manipulation and operations are central. Saxton et al. \cite{saxton2019analysing} created a dataset of math questions across various topics and observed that LSTM models barely memorized the question-answer pairs and could not generalize. Dave et al. \cite{dave2021math} used the same dataset to create distractors for math multiple-choice questions. Mistry et al. \cite{mistry2022primer} summarized the specialized architectures developed to solve arithmetic tasks, highlighting gaps in the robustness, compositionality, and interpretability of these models. The development of large language models (LLMs) has spurred significant interest in their mathematical and symbolic capabilities. Research has increasingly focused on training LLMs specifically for mathematical tasks, aiming to enhance their proficiency in symbolic reasoning. LLM models like Deepseek \cite{bi2024deepseek}, LLemma \cite{azerbayev2023llemma}, and Metamath \cite{yu2023metamath} are specially trained with math datasets. However, studies have shown that neural networks, including LLMs, often fail to genuinely learn to solve mathematical and symbolic tasks. Dziri et al. \cite{dziri2024faith} investigated the limitations of LLMs across three problems: multiplication, puzzle solving, and dynamic programming. They tested GPT-3, GPT-3.5, and GPT-4 models using zero-shot, few-shot, and fine-tuning techniques, revealing a significant decline in performance as the sample size increased. Additionally, Frieder et al. \cite{frieder_chatgpt} demonstrated that GPT-4's mathematical proficiency was far below graduate-level standards. Previous research on the mathematical proficiency of large language models (LLMs) has primarily focused on a broad set of math problems and reported findings based on static datasets such as GSM 8K \cite{cobbe2021training} and MATH \cite{hendrycks2021measuring}. This approach does not uncover the underlying inconsistencies in LLMs' understanding of mathematics. In this work, we examine the consistency and precision of LLMs as the complexity of the tasks increases.

Through our analysis, we assess the performance of enterprise and open-source LLMs trained on text and mathematical datasets, testing them on five symbolic tasks. These tasks fall into context-free and context-sensitive categories on Chomsky's hierarchy. We discuss the bit complexity of encoding symbolic tasks as an automaton and a set of knowledge tuples. Our experiments show that while enterprise LLMs, with their higher parameters and curated datasets, outperform open-source models, all models exhibit similar performance trends that align with the knowledge-tuple encoding of symbolic tasks. Fine-tuning these models on symbolic tasks has little effect on their performance, indicating that LLMs do not learn the rules of symbol manipulation but rather encapsulate relationships in terms of tuples. The contrast in the number of parameters required to encode symbolic tasks suggests a need for developing larger models capable of learning automata rather than merely storing information.

\section{Background}
Ideally, a model should either encode mathematical tasks as knowledge tuples or learn to simulate a machine that can solve such tasks efficiently. In this section, we examine the parameters required by models using both approaches. We relate tasks to their complexity, derived from Chomsky's hierarchy, with detailed definitions provided in the appendix. Furthermore, prior work by Allen-Zhu et al. \cite{AllenZhu2024PhysicsOL} has studied the bit capacity of large language models. Following their setup, we use the conversion of 2 bits per parameter to estimate the number of parameters required.


\subsection{Addition of Sequence of Numbers}
\label{sec:addition}
A finite state machine can add two numbers. Addition of sequence of numbers can be represented by production rules $A \rightarrow tB; A \rightarrow t $, where$A$ is the sequence of numbers, $t$ is the sum of partial sequence and $B$ is the remaining sequence. It is clear from the construction that a non-deterministic PDA can solve this problem. At each step of operation, the nPDA computes the sum of two numbers of arbitrary but finite number of digits. Let the the number of digits be $n$ and $m$. Then the number of digits in the sum of two numbers is $max(n, m) + 1$. There are two ways a LLM can encode the the rules of addition : 1) Encoding the states and stack in the internal parameters of the model, 2) Encoding addition rules for each pair of numbers. Encoding the nPDA requires $log_2(\#(STATES) + MAXLEN(stack))$ bits which is considerable lower than encoding specific addition rules for numbers of arbitrary number of digits. 

\begin{proposition}
For the addition of two base $p$ numbers with finite digits $n$ and $m$, respectively, where $n \geq m$, the encoding of their sum requires at most $(2n + m + 1) \log_2 p$ bits. This account includes the possibility of a carryover in the addition, which may increase the length of the resulting number by one digit.
\end{proposition}
\begin{proof}
Consider two numbers $A$ and $B$ represented in base $p$, where $A$ has $n$ digits and $B$ has $m$ digits with $n \geq m$. Each digit in these numbers requires $\log_2 p$ bits to encode due to the need to distinguish between $p$ different values.

\textbf{Step 1: Encoding Individual Numbers.}
\begin{itemize}
    \item The number $A$ requires $n \log_2 p$ bits.
    \item The number $B$ requires $m \log_2 p$ bits.
\end{itemize}

\textbf{Step 2: Maximum Size of Sum.}
When adding $A$ and $B$, the maximum number of digits in the sum, $C = A + B$, could be $n+1$ (considering the possibility of a carryover from the most significant digit).

\textbf{Step 3: Encoding the Sum.}
The sum $C$ hence requires at most $(n+1) \log_2 p$ bits.

\textbf{Step 4: Total Encoding Requirement.}
Adding the encoding requirements together, the total number of bits needed is:
\[
n \log_2 p + m \log_2 p + (n + 1) \log_2 p = (2n + m + 1) \log_2 p.
\]
This calculation confirms that the sum of the encoding requirements for the numbers and their potential maximum sum, considering carryover, matches and proves the proposition.
\end{proof}



\begin{corollary}
Consider a sequence of \(N\) numbers, each of base \(10\) and with at most \(D\) digits. To completely memorize all possible addition operations within this sequence, a Large Language Model (LLM) must store at least \(\eta_A (3D_{\text{max}} - 2) \log_2 10\) bits of information, where:

\[
D_{\text{max}} = \left\lfloor 1 + \log_{10}\left((10^D - 1) \cdot N\right) \right\rfloor,
\]

and 

\[
\eta_A = \frac{\lambda (\lambda + 1)}{4}, \quad \quad \lambda = N(10^{D}-1)
\]

representing the total number of distinct pairwise addition operations that could be required to solve the sequence. Here, \(D_{\text{max}}\) represents the maximum number of digits in any number or partial sum within the given sequence, and \(\eta_r\) quantifies the total number of unique addition rules required for comprehensive solution strategies.
\end{corollary}
Table \ref{tab:limits} displays the maximum size of sequences that can be effectively solved by LLMs through the memorization of all necessary addition rules.


\subsection{Multiplication}
\label{sec:multiplication}
Multiplicating two numbers of arbitrary digits is a context-sensitive task. It requires a finite state machine augmented with a finite tape. For the input expression of form $a \times b$ on the input tape, the tape pointer reads the digits of $a$ from right to left. For digit $d$ with place value $v$ in $a$, the pointer moves right to the output part of the tape and increments it $d*v*b$ times. The pointer moves back to the digit, sets it to $\text{<BLANK>}$, and moves to the digit on the left.  

\begin{proposition}
An automaton, which includes a finite state controller and a finite length tape, capable of computing the multiplication of an \( n \)-digit number and an \( m \)-digit number, requires encoding with
\[
\log_2 6 + 2 \cdot O(n+m) \cdot \log_2 p \text{ bits}.
\]
\end{proposition}

\begin{proof}
To compute the multiplication of two numbers, each represented by \( n \) and \( m \) digits respectively, consider an automaton designed as follows:

\textbf{State Design:}
    The finite state controller is structured with six distinct states to manage the computation:
    \begin{enumerate}
        \item Two states are dedicated to reading the digits of the numbers, specifically one state for each number (\( n_1 \) and \( n_2 \)).
        \item One state is allocated for writing the output of the multiplication.
        \item Three additional states facilitate the movement of the tape head to the appropriate positions corresponding to the digits of the two numbers and the resulting output.
    \end{enumerate}
    
\textbf{Tape Complexity:}
    The tape must accommodate the digits of both numbers and the resultant product. The maximal length of the product of two numbers, each with \( n \) and \( m \) digits, is \( n + m \) digits. Therefore, the tape's length is approximately \( O(2 \cdot (n + m)) \) to account for both input numbers and their maximal output size.
    
\textbf{Bit Encoding:}
    The number of bits required to encode the state machine is determined by the logarithm of the number of states, giving \( \log_2 6 \) for the states. The tape encoding requires \( 2 \cdot O(n+m) \cdot \log_2 p \) bits, where \( \log_2 p \) bits are needed to encode each digit of the input and output numbers, assuming the numbers are in base \( p \).

Combining these components, the total encoding requirement for the automaton is thus:
\[
\log_2 6 + 2 \cdot O(n+m) \cdot \log_2 p \text{ bits}.
\]
This completes the proof.
\end{proof}




\begin{proposition}
\label{prop:mult_tuple}
Given two integers \(a\) and \(b\), where \(a\) is an \(n\)-digit number and \(b\) is an \(m\)-digit number in base \(p\), a multiplication rule that expresses the relationship \(a \times b = \gamma\) can be encoded as a knowledge tuple \((a, b, \gamma)\). This tuple can be stored using at most \(2(n+m)\log_2 p\) bits.
\end{proposition}

\begin{proof}
The multiplication of \(a\) and \(b\) results in \(\gamma\), where \(\gamma\) can have at most \(n + m\) digits (considering the maximum carry-over in base \(p\)):

\begin{itemize}
    \item Each digit of \(a\) and \(b\) can be encoded using \(\log_2 p\) bits because each digit represents a value from 0 to \(p-1\).
    \item Therefore, \(a\) requires \(n \cdot \log_2 p\) bits and \(b\) requires \(m \cdot \log_2 p\) bits.
    \item The product \(\gamma\) can be as large as \(n + m\) digits, thus requiring at most \((n + m) \cdot \log_2 p\) bits.
\end{itemize}

The knowledge tuple \((a, b, \gamma)\) encapsulates the complete multiplication expression and therefore combines the storage requirements of \(a\), \(b\), and \(\gamma\). Thus, the total bit requirement is the sum of the bits needed to store \(a\), \(b\), and \(\gamma\):
\[
2(n+m)\log_2 p \text{ bits},
\]
assuming the maximum possible size for \(\gamma\) and ignoring potential savings from compressing common information between \(a\), \(b\), and \(\gamma\). This completes the proof.
\end{proof}

\begin{corollary}
Given all possible tuples \((a, b, \gamma)\) where \(a\) and \(b\) are numbers in base \(10\) with a maximum of \(D_{\text{max}}\) digits, these tuples can be encoded in a Large Language Model (LLM) using at most 
\[
4D_{\text{max}} \cdot 10^{2D_{\text{max}}} \cdot \log_2 10 \text{ bits},
\]
where \(D_{\text{max}}\) is the maximum number of digits in any number among \(a\), \(b\), and \(\gamma\).
\end{corollary}

\begin{proof}
To encode each of the tuples \((a, b, \gamma)\) where \(a\) and \(b\) are up to \(D_{\text{max}}\) digits:
\begin{enumerate}
    \item The total number of possible values for \(a\) or \(b\) is \(10^{D_{\text{max}}}\), since each digit can range from 0 to 9.
    \item The product \(\gamma = a \times b\) can have at most \(2D_{\text{max}}\) digits (considering the worst-case scenario of multiplication).
    \item Thus, the number of possible tuples \((a, b, \gamma)\) is \(10^{D_{\text{max}}} \times 10^{D_{\text{max}}} \times 10^{2D_{\text{max}}} = 10^{4D_{\text{max}}}\).
    \item Each tuple then requires encoding that can be estimated by calculating the total number of bits to represent each digit of \(a\), \(b\), and \(\gamma\) in binary:
    \[
    4D_{\text{max}} \cdot \log_2 10 \text{ bits per digit} \times 10^{2D_{\text{max}}} \text{ possible tuples}.
    \]
\end{enumerate}
This provides upper bound on the bit requirement to encode tuples in an LLM.
\end{proof}


\subsection{Symbolic Counting}
\label{sec:counting}
Counting the frequency of a character $c$ in a given string $s$ is well solved using an automaton equipped with a counter and a register. Let $V$ be the set of all possible characters s.t. $c \in V$ and $s \in V^*$. First, the automaton reads $c$ and stores it in the register. The finite state controller then reads characters $c' \in s$ step by step and increments the counter when $c' = c$. 

\begin{proposition}
    An automaton capable of counting the frequency of a character in a given string can be encoded in at most $log_23 + log_2|V| + \lfloor 1 + \log_{p} N_{max}\rfloor*\log_2p$ bits, where $N_{max}$ is the maximum length of an input string in base $p$
\end{proposition}

\begin{proof}

\textbf{Step 1 : Encoding the states}: The above discussed automaton can function with $3$ states. The automaton starts in state $q0$, reads the character $c$ into the register, and moves to the next state ($q1$). The automaton now reads the string, one character at a time. If the input character matches the character in the register, the automaton moves into state $q2$. Otherwise, it remains in $q1$. The states $\{q0, q1 , q2\}$ can be encoded in $\log_23$ bits. 

\textbf{Step 2 : Encoding the register} : Since the register holds one character $c\in V$, it can be encoded in $log_2|V|$ bits.

\textbf{Step 3: Encoding the counter}: The maximum value the counter would need to hold is the length of the longest string. Thus counter can be encoded in $ \lfloor 1 + \log_{p} N_{max}\rfloor*\log_2p$ bits.
\end{proof}

LLMs can mimic this behavior by encoding knowledge tuples $(c, s, n)$

\begin{proposition}
The knowledge tuple $(c, s, n)$, where $n$ is the frequency of $c$ in $s$ can be encoded in atmost $(N_{max}+ 1)log_2|V| + \lfloor 1 + \log_{p} N_{max}\rfloor*\log_2p$ bits.
\end{proposition}

\begin{corollary}
An LLM can encode the entire sample space of $(c, s, n)$ knowledge tuples of maximum string length $N_{max}$(in base $10$) in atmost  $|V|^{N_{max}+ 1} \Delta$, where 

$$ \Delta = (N_{max}+ 1)log_2|V| + \lfloor 1 + \log_{10} N_{max}\rfloor*\log_210$$

\end{corollary}

Encoding the entire space of possible string combinations is a gargantuan task. It's very inefficient, even for massive LLMs. Especially when counting can be reduced to the addition of a sequence of ones and zeros, where $1$ represents a character match and $0$ represents no match. This reduces symbolic counting into a two-step process:

\textbf{Step 1}: Using attention mechanism to reduce the character matching problem to a binary sequence that represents a match. This can be encoded in atmost $N_{max} + (N_{max}+ 1)log_2|V|$ bits. 

\textbf{Step 2} : Perform sum of sequence on the sequence of $1$s and $0$s. There will be total of $2N_{max}$ knowledge tuples to encode addition in form $(n, 1, n+1)$ and $(n, 0, n)$. These can be encoded in almost $2N_{max}*(1 + 2 \lfloor 1 + \log_{10} N_{max}\rfloor \log_210 )$ bits.

Thus, the bits required by LLMs to encode symbolic counting are: 
$$N_{max} + (N_{max}+ 1)log_2|V| + 2N_{max}*(1 + 2 \lfloor 1 + \log_{10} N_{max}\rfloor \log_210 )$$

\subsection{Parameter Estimation}
Based on the estimates discussed above, we now find out the computational capacity of large language models with respect to the tasks of addition, multiplication, and counting. Table \ref{tab:limits} reflects the bound on input size for given tasks for a model of 7B and 180B parameters, respectively. These bounds are much closer to the performance that we observe in our experiments. Most LLMs are trained on multiple tasks and are primarily focused on modeling language rather than symbolic tasks. Hence, our results show that LLM performance degrades for sample sizes that are way smaller than the bounds discussed.

\begin{table}[h]
\centering
\begin{tabular}{llll}
\hline
Task              & Size Parameter                    & 7B  & 180B  \\ \hline
Sum of Sequence   & Sequence Length                   & 363 & 1660  \\
Multiplication    & $\#digits(n_1) + \#digits(n_2)$   & 8   & 9     \\
Symbolic Counting & String Length                     & 1e8 & 2.5e9 \\ \hline
\end{tabular}
\caption{Maximum input size of the sample for the three tasks for 7B and 180B parameter models. The tasks are encoded as knowledge tuples.}
\label{tab:limits}
\end{table}
Table \ref{tab:pbound} shows the number of parameters required for the two encoding systems for a fixed input size on the given tasks. It is observed that, if neural network models could simulate the machines to solve the task, they would require considerably less number of parameters. Current LLMs have parameters more aligned with the knowledge tuple encoding system.

\begin{table}[h]
\begin{tabular}{lcccc}
\hline
\multicolumn{1}{c}{Task} & \begin{tabular}[c]{@{}c@{}}Size Parameter\\  Name\end{tabular} & \begin{tabular}[c]{@{}c@{}}Parameter \\ Value\end{tabular} & \begin{tabular}[c]{@{}c@{}}Machine \\ Encoding\end{tabular} & \begin{tabular}[c]{@{}c@{}}Knowledge Tuple\\  Encoding\end{tabular} \\ \hline
Sum of Sequence          & Sequence Length                                                & 100                                                        & 8                                                           & 407M                                                                \\
Multiplication           & $\#digits(n_1) + \#digits(n_2)$                              & 8                                                          & 28                                                          & 2B                                                                  \\
Symbolic Counting        & String Length                                                  & 1000                                                       & 11                                                          & 30,929                                                              \\ \hline
\end{tabular}
\caption{Comparison of the number of parameters required in the model for machine encoding vs knowledge tuple encoding.}
\label{tab:pbound}
\end{table}
\section{Experimental Setup}
Fundamental mathematics operations, like addition and multiplication, are well understood as symbolic operations. Numbers can be represented as composite symbols composed of a pre-defined set of symbols, i.e., $0-9$ numerals. We analyze the performance of large language models on varying difficulty of the tasks modeled as the input size. All models are tested on $100$ samples for each difficulty level. We keep a consistent seed for data generation to ensure a fair comparison. 

\subsection{LLM Models}
%

We conduct our experiments on eight LLM models, including four enterprise models and four open-source models. The enterprise models are GPT-3.5 (gpt-3.5-turbo-0125), GPT-4 (gpt-4-0125-preview), Gemini (gemini-pro-1.0), and Claude (claude-3-haiku-20240307). The open-source models analyzed are Llamma 2 (13B), Llemma (7B), Deepseek (7B), and MetaMath (7B). All open source models except Llamma-2 13B are pre-trained on math tasks. Llemma 7B is Llamma-7B based model fine tuned on math datasets. MetaMath is fine-tuned on Llemma-7B. Deepseek is trained on math datasets like OpenWebMath. 

\textbf{Model query parameters } 
We keep the maximum token size of 4096 for all models. All models are queried with temperature of $0$ to maximize determinism in the generated text. Models are not given access to any external API's, python interpretter,  documents or databases, to ensure analysis of standalone LLM performance.  No examples are provided in the prompt, and single prompt query is used for each input sample. Zero-shot COT method is used to allow models to navigate the problems without any help.

\textbf{Prompting}
All prompts are composed of two components: \textit{system prompt} and \textit{user prompt}. Gemini API does not have separate fields for system and user prompts, hence we concatenate both. System prompt is composed of the following components:
\begin{enumerate}
    \item \textbf{Identity} - e.g. for sum of sequence, the prompt stats with \textit{"You are number crunching machine..."}
    \item \textbf{Input type} - the input type can be \textit{sequence of numbers}, \textit{arithmetic expression} etc.
    \item \textbf{Command} - instruction on what to do with the input. e.g. \textit{calculate the sum}
    \item \textbf{Constraint} [Optional] - Any extra guideline or constraint for the model, e.g. \textit{Treat the upper and lower case characters as separate}
    \item \textbf{Zero-Shot COT} - To help models present the intermediate thoughts and reasoning without providing any example, we use Zero-shot-COT method. This is done by appending \textit{Let's think step by step}, at the end of the prompt.
\end{enumerate}

The user prompt has two parts : 1) \textit {input identifier} and 2) \textit{input sequence}

\subsection{Symbolic Tasks}

\begin{table}[h]
\centering
\begin{tabular}{llc}
\hline
\multicolumn{1}{c}{Task} & \multicolumn{1}{c}{Categorization by Chomsky’s Hierarchy}                                                                                                                         & Machine                                                              \\ \hline
Sum of Sequence          & \begin{tabular}[c]{@{}l@{}}Sum of sequence od multi-digit numbers \\ can be constructed as:\\ $A \rightarrow tB ; A\rightarrow t$ \\ Hence it is a context free task\end{tabular} & Pushdown Automata                                                    \\ \hline
Multiplication           & \begin{tabular}[c]{@{}l@{}}Multiplication of two multi-digit numbers \\ is a context-sensitive task as \\ it requires a bounded tape to \\ perform the operation.\end{tabular}    & \begin{tabular}[c]{@{}c@{}}Linearly Bounded \\ Automata\end{tabular} \\ \hline
Symbolic Counting        & \begin{tabular}[c]{@{}l@{}}Counting the matches of a character \\ in a string requires a counter with a FSM,\end{tabular}                                                         & Counter Automata                                                     \\ \hline
Modular-$10$ Arithmetic  & \begin{tabular}[c]{@{}l@{}}A modular arithmetic expression \\ with brackets require a stack \\ to store intermediate results and \\ is thus context-free task\end{tabular}        & Pushdown Automata                                                    \\ \hline
Decimal Arithmetic       & \begin{tabular}[c]{@{}l@{}}Decimal Arithmetic requires a tape to \\ perform multiplication and thus is \\ context-sensitive.\end{tabular}                                         & \begin{tabular}[c]{@{}c@{}}Linearly Bounded \\ Automata\end{tabular} \\ \hline
\end{tabular}
\caption{Categorization of Symbolic Tasks in Chomsky's Hierarchy}
\label{tab:chomsky}
\end{table}

Table \ref{tab:chomsky} shows the anchoring of symbolic tasks in Chomsky's Hierarchy. Sections \ref{sec:addition}, \ref{sec:multiplication} and \ref{sec:counting} provide the machine encoding for addition, multiplication and counting tasks, respectively. Modulo-$10$ arithmetic simplifies addition and multiplication to single digit numbers, but add the complexity of brackets placing them in context-free category. 
Decimal Arithmetic provides a perspective on the precision handling capability of LLMs, sourcing its theoretical framework from addition and multiplication tasks. We define various tasks based on complexity as follows:
\begin{enumerate}
 \item \textbf{Sum of Sequence}
The model is tasked with computing the sum of the sequence of numbers. The numbers in the sequence are sampled uniformly in the range $10$—$99$. Task difficulty is modeled as the length of input sequences. We test the sequence of sizes $5, 10, 20, 50, 100, 200,$ and $500$. 

 \item \textbf{Modulo-10 Arithmetic}
While Deletang et al. \cite{deletang2022neural} used modulo-$5$ arithmetic to benchmark RNN and Transformer models. We follow a similar approach and analyze LLMs' performance on modulo-$10$ arithmetic expressions. The expressions are composed of the following operations: addition, subtraction, multiplication, and brackets. The difficulty of the task is modeled as the length of expression. We test LLMs with expressions of length $5, 10, 20, 50, 100, 200, 500, 1000,$ and $2000$. 

 \item \textbf{Decimal Arithmetic}
We test the precision handling capacity of LLMs by introducing decimal numbers in arithmetic expressions. The difficulty in task is modeled as the number of digits in the fractional part of decimal, while keeping one digit in whole number part. We test the expressions with $1, 2, 3,4,5,$ and $6$ fractional digits,

 \item \textbf{Multiplication}
This tests LLM models on the multiplication of $2$ whole numbers. The difficulty in task is modeled as the sum of number of digits in multiplier and the multiplicand. Multiplication with combined digits of $2, 4, 6, 8, 10,$ and $12$ are tested.

 \item \textbf{Symbolic Counter}
The LLMs are tested on their ability to count the number of character matches in a string. The difficulty of task is modeled as the length of string.  We test LLMs on string lengths on $5, 10, 20, 50, 100, 200, 500, 1000,$ and $2000$
\end{enumerate}

\subsection{Fine Tuning}
We fine-tune GPT3.5 model on sum of sequence and symbolic counting tasks. For sum of sequence we fine-tune specifically for sequence lengths $5, 10, 20, 50$ and  $100$ with $300$ samples for each seq length. The validation set is composed of 50 samples for each length. Finally the model is tested on all sequence lengths i.e,  $5, 10, 20, 50, 100, 200,$ and $500$ with $100$ samples for each length. Similarly, for the symbolic counting tasks, the model is fine-tuned with strings of length $50$, $100$, $200$, and  $500$, with $300$ for each length. The model is tested for all lengths with $100$ samples each.
\section{Results and Discussion}
Figures \ref{fig:sos} - \ref{fig:ft} show the performance of LLMs on the five symbolic tasks. In all cases, we can see that performance starts degrading as the difficulty of input samples increases. By comparing these results with tables \ref{tab:limits} and \ref{tab:pbound} we can see that performance starts degrading for far simpler examples than anticipated. This indicates that LLMs have not learned to solve symbolic problems but rather have a notion of input-output pairings. From the figures, we can see that while GPT3.5 and Claude Haiku models perform very similarly in all tasks, Gemini is able to sustain performance longer than the other models on the sum of sequence task while it is usually behind other models in the rest of the tasks. GPT4 outperforms other models in modulo-$10$ arithmetic, decimal arithmetic, and multiplication tasks while lagging behind in other tasks. Claude Haiku performs fairly consistently in all tasks, being the smallest of all enterprise models. Out of all models, Lemma 7B is the worst-performing model across all tasks. Deepseek and MetaMath perform slightly better due to their math pre-training.

\begin{figure}
     \centering
     \begin{subfigure}[b]{0.4\textwidth}
         \centering
         \includegraphics[width=\textwidth]{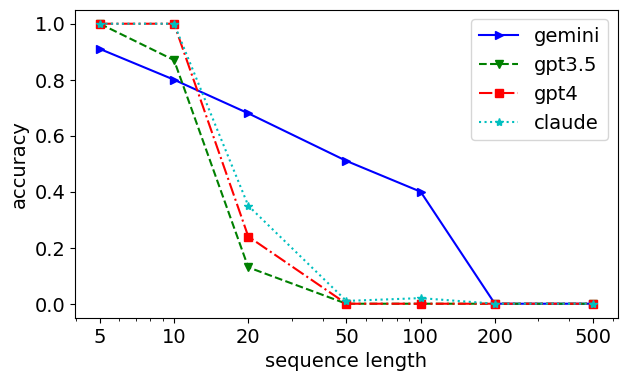}
         \caption{Enterprise LLMs}
         \label{fig:ent_sos}
     \end{subfigure}
     \begin{subfigure}[b]{0.4\textwidth}
         \centering
         \includegraphics[width=\textwidth]{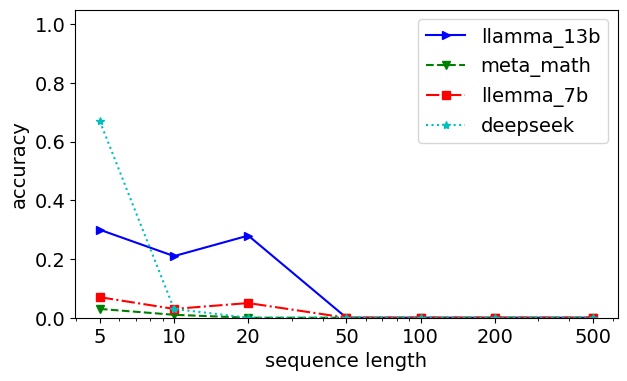}
         \caption{Open Source LLMs}
         \label{fig:os_sos}
     \end{subfigure}
 
     \caption{Performance of LLMs on Sum of Sequence task}
     
     \label{fig:sos}
\end{figure}


\begin{figure}
     \centering
     \begin{subfigure}[b]{0.4\textwidth}
         \centering
         \includegraphics[width=\textwidth]{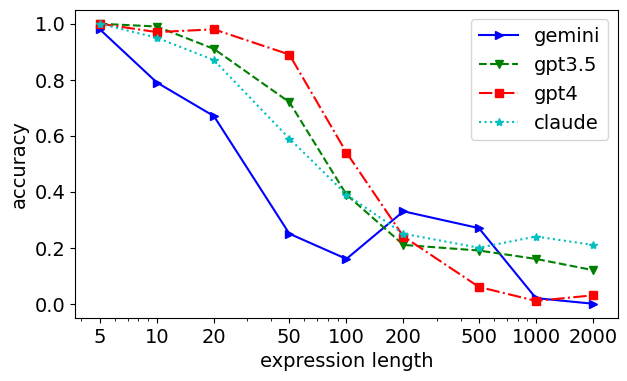}
         \caption{Enterprise LLMs}
         \label{fig:ent_arithmetic}
     \end{subfigure}
     \begin{subfigure}[b]{0.4\textwidth}
         \centering
         \includegraphics[width=\textwidth]{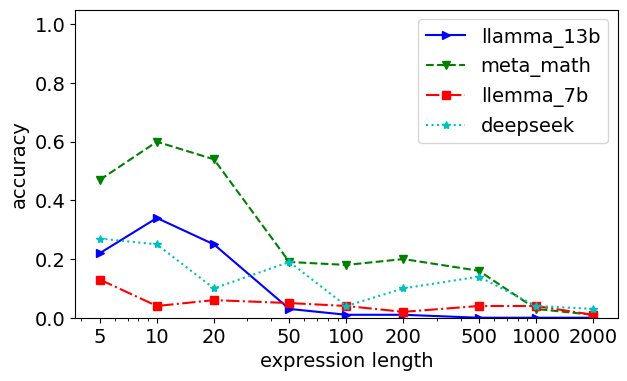}
         \caption{Open Source LLMs}
         \label{fig:os_arithmetic}
     \end{subfigure}
     \caption{Performance of LLMs on Modulo 10 Arithmetic task }
     \label{fig:arithmetic}
\end{figure}


\begin{figure}
     \centering
     \begin{subfigure}[b]{0.4\textwidth}
         \centering
         \includegraphics[width=\textwidth]{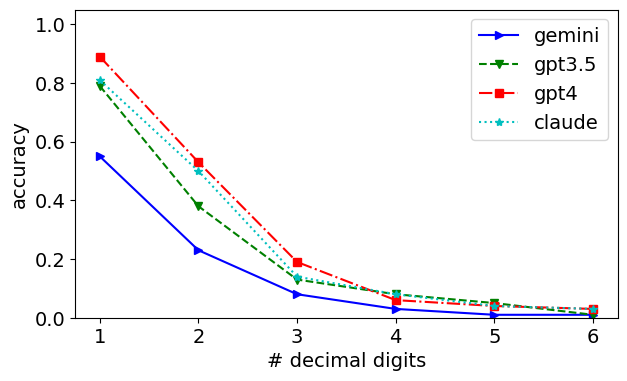}
         \caption{Enterprise LLMs}
         \label{fig:ent_da}
     \end{subfigure}
     \begin{subfigure}[b]{0.4\textwidth}
         \centering
         \includegraphics[width=\textwidth]{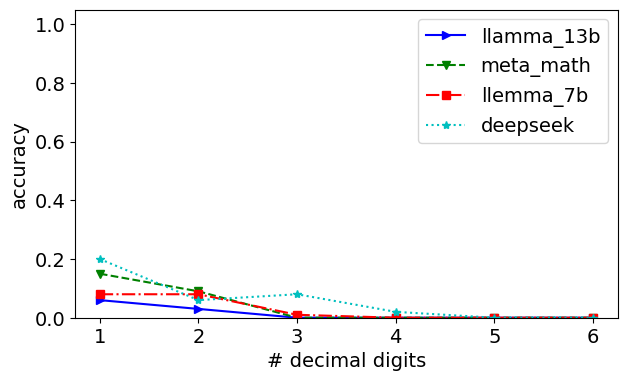}
         \caption{Open Source LLMs}
         \label{fig:os_da}
     \end{subfigure}
     \caption{Performance of LLMs on Decimal Arithmetic task }
     \label{fig:da}
\end{figure}


\begin{figure}
     \centering
     \begin{subfigure}[b]{0.4\textwidth}
         \centering
         \includegraphics[width=\textwidth]{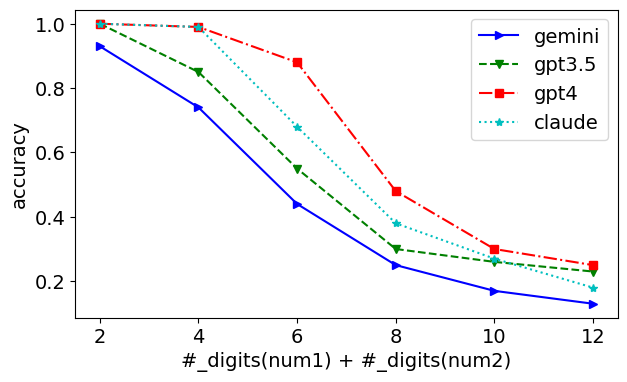}
         \caption{Enterprise LLMs}
         \label{fig:ent_mult}
     \end{subfigure}
     \begin{subfigure}[b]{0.4\textwidth}
         \centering
         \includegraphics[width=\textwidth]{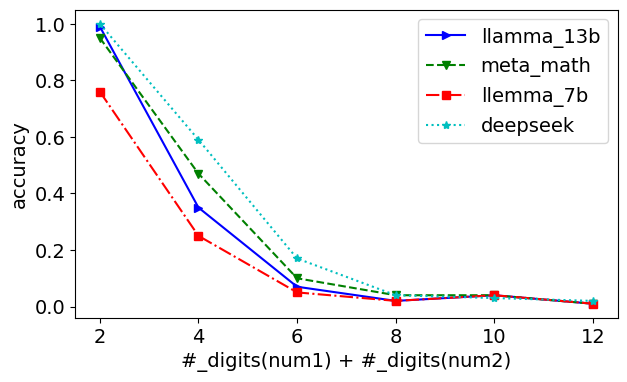}
         \caption{Open Source LLMs}
         \label{fig:os_mult}
     \end{subfigure}
     \label{fig:mult}
     \caption{Performance of LLMs on Multiplication task }
\end{figure}


\begin{figure}
     \centering
     \begin{subfigure}[b]{0.4\textwidth}
         \centering
         \includegraphics[width=\textwidth]{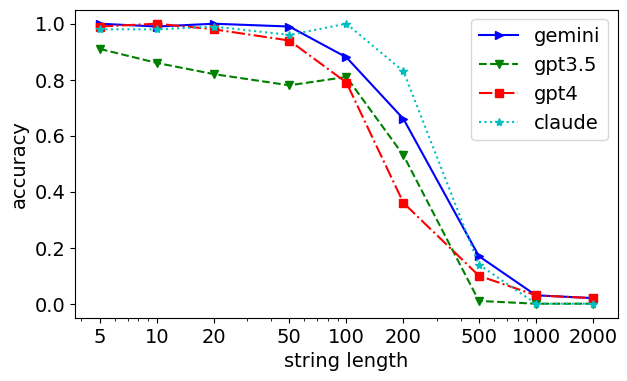}
         \caption{Enterprise LLMs}
         \label{fig:ent_scount}
     \end{subfigure}
     \begin{subfigure}[b]{0.4\textwidth}
         \centering
         \includegraphics[width=\textwidth]{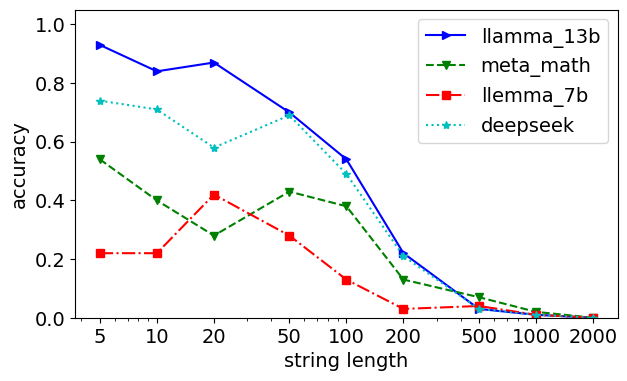}
         \caption{Open Source LLMs}
         \label{fig:os_scount}
     \end{subfigure}
     \caption{Performance of LLMs on Symbolic Counter task }
     \label{fig:scount}
\end{figure}

\begin{figure}
     \centering
     \begin{subfigure}[b]{0.4\textwidth}
         \centering
         \includegraphics[width=\textwidth]{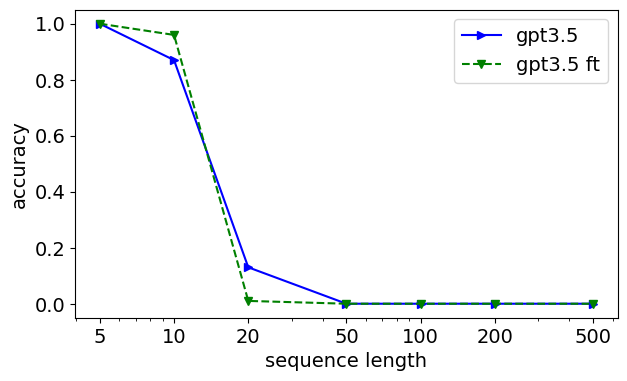}
         \caption{GPT 3.5 fine-tuned on Sum of Sequence Task}
         \label{fig:ft_sos}
     \end{subfigure}
     \begin{subfigure}[b]{0.4\textwidth}
         \centering
         \includegraphics[width=\textwidth]{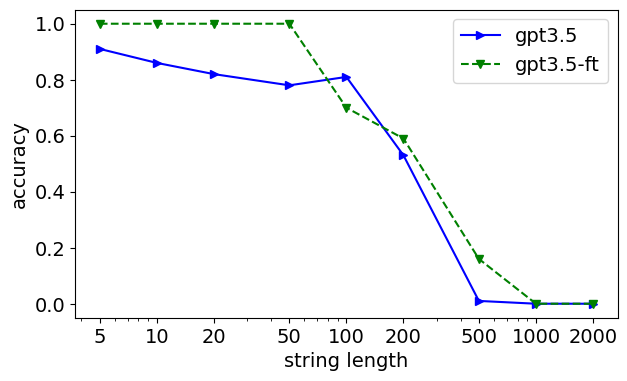}
         \caption{GPT 3.5 Fine-Tuned on Symbolic Counter task}
         \label{fig:ft_scount}
     \end{subfigure}
     
     \caption{Effect of fine-tuning on GPT 3.5 performance on symbolic tasks}
     \label{fig:ft}
\end{figure}

\textbf{Enterprise and Open Source Models}
Multiplication is the only task where open-source models and enterprise models start with similar performances. In symbolic counting, the Llamma 13B model performs similarly to enterprise models for strings of length $5$. In all cases, the performance of open-source models degrades much faster than that of enterprise models. In the other three tasks, open-source models perform far below enterprise models, even for simpler cases. An argument could be made that open-source models are smaller than the enterprise models tested in this work. This further indicates that performance is highly dependent on the models' size, making the knowledge-tuple encoding hypothesis much stronger. 

\textbf{Math trained models}
Out of the three math-trained models, Deepseek gives the best overall performance. However, even math-trained models are not able to sustain their performance on larger examples. Symbolic counting is the only task where the Llamma 13B model outperforms math-trained models. All math-trained models have the number of parameters of Llamma 13B. 

\textbf{Effect of fine tuning}
Figure \ref{fig:ft} shows the comparison of vanilla and fine-tuned GPT3.5 models on the sum of sequence and symbolic counting tasks. We can observe that fine-tuning has helped performance in smaller examples but made larger examples a little worse. Overall, all fine-tuning and closing performances follow vanilla performance, and there is no improvement in the generalization of the tasks.

\section{Conclusion}
In this work, we analyze the performance of enterprise and open-source LLMs trained on text and mathematical datasets on five symbolic tasks. The tasks occupy context-free and context-sensitive categories on Chomsky's hierarchy. We discussed the bit complexity of encoding symbolic tasks as an automaton and set of knowledge tuples. Our experiments show that although enterprise LLMs perform better than open-source models catering to higher parameters and curated dataset, all models show similar trends in performance and align with knowledge-tuple encoding of symbolic tasks. Fine tuning on these tasks as little effect on the performance of LLMs. This clearly shows that LLMs are not learning the rules of symbol manipulation but rather encapsulate relationships in terms of tuples. The contrast in the number of parameters required to encode symbolic tasks calls for a push in the direction of large models capable of learning automaton rather than storing information. 

%
%
%
%
%
%


\bibliography{references}
\clearpage
\appendix
\section{Chomsky's Hierarchy}
Chomsky \cite{ChomskyFormalProperties} \cite{chomskyThreeModels} established a hierarchy (see Figure \ref{fig:chomsky}) delineating formal grammars according to their complexity and the corresponding computational devices capable of recognizing languages generated by these grammars. These grammars are classified from Type-3 to Type-0, with Type-0 representing the highest level of complexity.

\begin{figure}[h]
\centering
\includegraphics[width=0.6\textwidth]{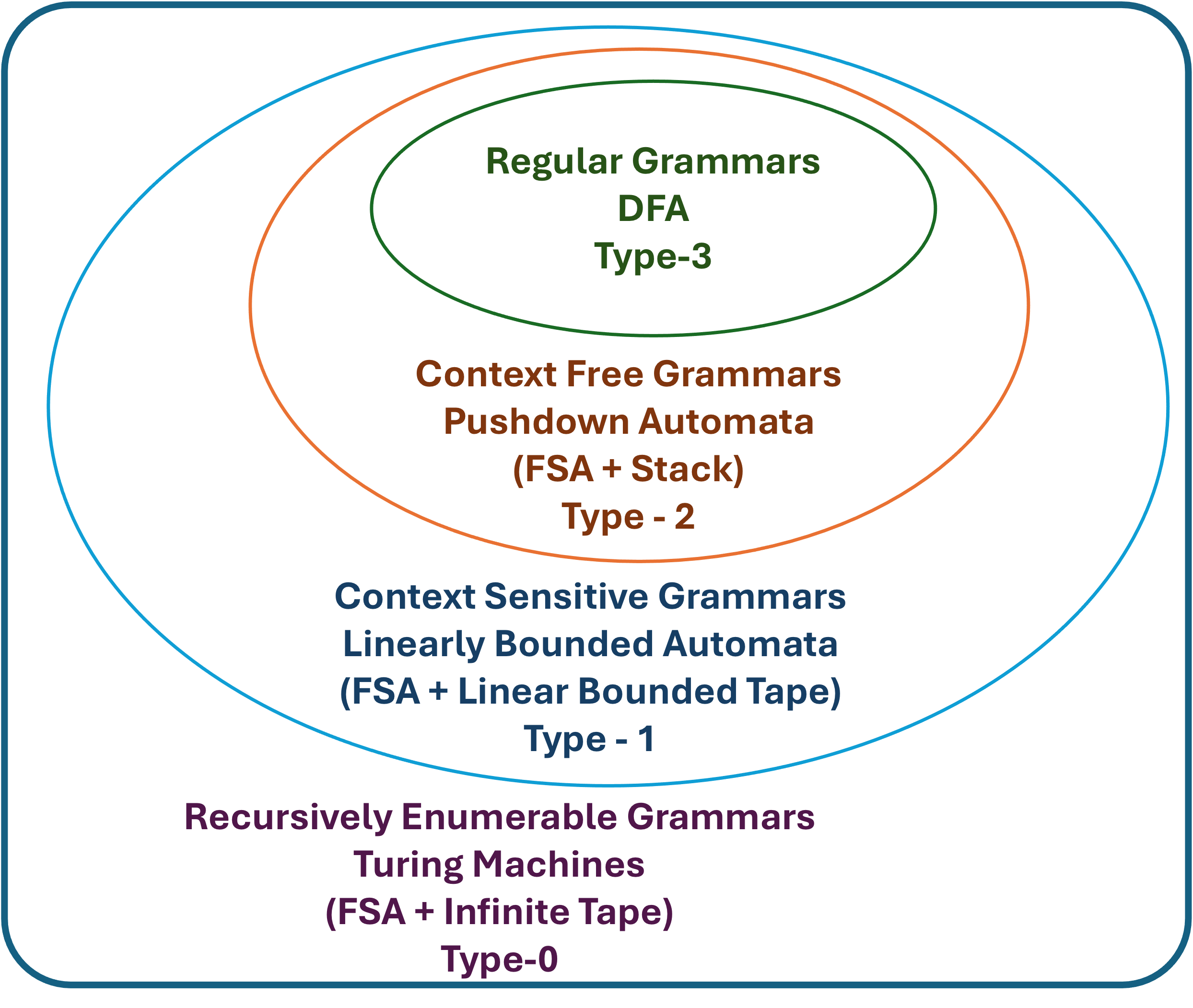}
\caption{Chomsky's Hierarchy}
\label{fig:chomsky}
\end{figure}

\begin{definition}
A \textbf{Deterministic Finite (State) Automaton} is a $5$-tuple, $(Q, \Sigma, \delta, q_0, F)$, where $Q$ is the set of all states, $q_0 \in Q$ is the initial state. $\Sigma$ is the set of alphabets such that all input strings are a subset of $\Sigma^*$. $ \delta : Q \times \Sigma \rightarrow Q$  are a set of transition rules. $F \subseteq Q$ are accepting states.
\end{definition}

 A DFA reads input one symbol at time and updates its state determined by the transition rules $\delta$. The input string is accepted if the DFA reaches a state $q \in F$ at the end of the input.

\begin{definition}
     A \textbf{Non-Deterministic Finite (State) Automaton} is a $5$-tuple similar to DFA with a non-deterministic transition function $\delta : Q \times (\Sigma \cup \{ \epsilon \}) \rightarrow \mathcal{P}(Q)$, where $\mathcal{P}(Q)$ is a power set of all automaton states.
\end{definition}

 All non-deterministic FSA can be reduced to their deterministic counterparts without the loss of expressivity. Finite State Automata can accept all regular languages.

\begin{definition}
     A \textbf{Finite State Transducer} (FST) is a finite state machine with an output tape. Formally it is a $6$-tuple, $(Q, \Sigma, \Gamma, q_0, F, \delta)$ where $\Lambda$ is the set of output alphabet, and $\delta : Q \times \{\Sigma \cup \epsilon\} \rightarrow  Q \times \{\Lambda \cup \epsilon\}$. 
\end{definition}

\begin{definition}
    A \textbf{Pushdown Automaton} is a $7$-tuple, $(Q, \Sigma, S, \delta, q_0, z, F)$, where $S$ is a set of stack symbols and $z \in S$ is initial stack symbol. Transition function is defined as $\delta :  Q \times \{\Sigma \cup \epsilon\} \times S \rightarrow Q \times S^*$
\end{definition}

\begin{definition}
    A \textbf{Pushdown Transducer} is a pushdown automata with an output tape. Specifically it is a $8$-tuple, $(Q, \Sigma, \Lambda, S, \delta, q_0, z, F)$, where $\Lambda$ is a set of output symbols. All other symbols carry the same meaning as a pushdown automata. The transition function is defined as $\delta :  Q \times \{\Sigma \cup \epsilon\} \times S \rightarrow Q \times \{\Lambda \cup \epsilon\} \times S^*$
\end{definition}

\begin{definition}
    A \textbf{Turing Machine} is a $7$-tuple, $(Q, \Gamma, \beta, \Sigma, \delta, q_0, F)$. $\Gamma$ is the set of all tape alphabets. $\beta \in \Gamma$ is blank symbol, which is default on any on-empty cell on the tape. Input symbols are part of tape symbols, $\Sigma \subseteq \Gamma \ {b}$, and transition function is defined as $\delta : Q \times \Gamma \rightarrow Q \times \Gamma \times \{N, L, R\}$. Here $\{N, L, R\}$ are \textit{no shift}, \textit{left shift} and \textit{right shift} of the tape pointer respectively.
\end{definition}

A FST can generate all regular languages and similarly, a pushdown transducer can generate all context free languages. A linearly bounded automaton is a Turing Machine with a finite size tape.

\section{Bit Encoding of a Number}
\begin{proposition}
Given a number \(N\) in base \(p\), the model requires 
\[
\left\lfloor 1 + \log_p(N) \right\rfloor \cdot \log_2 p \text{ bits}
\]
to encode it in symbolic form.
\end{proposition}

\begin{proof}
The number of digits \(d\) required to represent a number \(N\) in base \(p\) can be determined by the formula:
\[
d = \left\lfloor 1 + \log_p(N) \right\rfloor.
\]
This expression arises because the largest number representable with \(d\) digits in base \(p\) is \(p^d - 1\), and thus:
\[
p^{d-1} \leq N < p^d.
\]
Taking the logarithm base \(p\) of all parts of this inequality results in:
\[
d - 1 \leq \log_p(N) < d.
\]
Adding 1 to each part and then taking the floor function provides the correct number of digits:
\[
d = \left\lfloor 1 + \log_p(N) \right\rfloor.
\]

Each digit in base \(p\) requires \(\log_2 p\) bits for encoding since each digit can take any of \(p\) different values, and \(\log_2 p\) bits are sufficient to uniquely identify each value. Therefore, the total number of bits required to encode \(N\) is the product of the number of digits and the number of bits per digit:
\[
\left\lfloor 1 + \log_p(N) \right\rfloor \cdot \log_2 p.
\]
This completes the proof.
\end{proof}


\clearpage
\section{Prompt Examples}

\begin{figure}[ht]
     \centering
     \includegraphics[width=\textwidth]{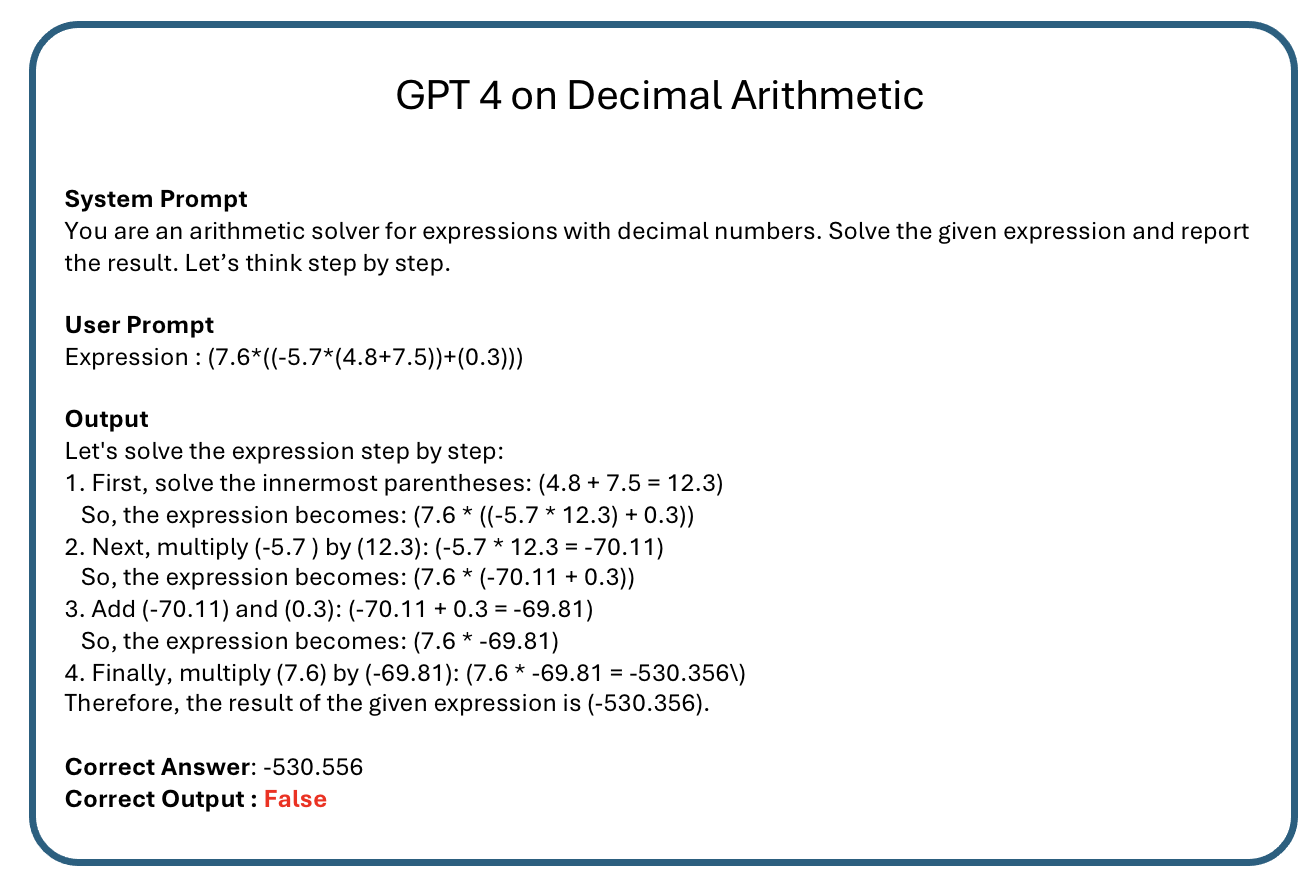}
\end{figure}

\begin{figure}[ht]
     \centering
     \includegraphics[width=\textwidth]{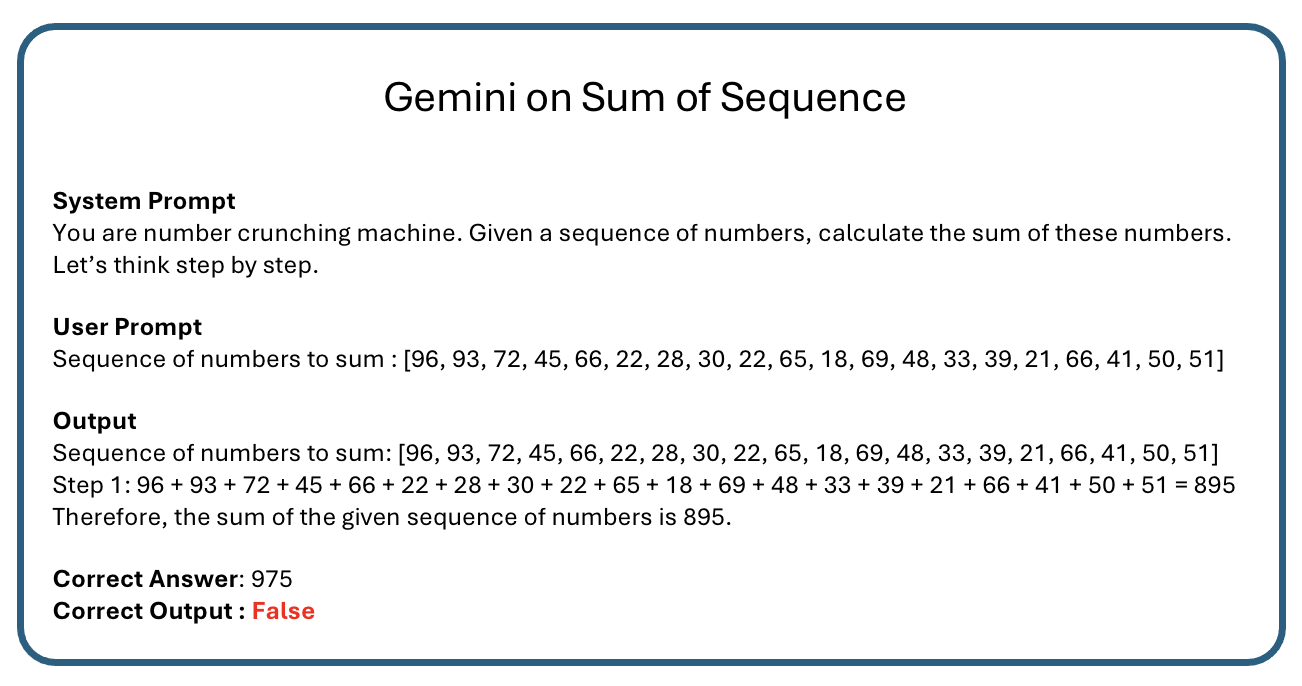}
\end{figure}

\begin{figure}[ht]
     \centering
     \includegraphics[width=\textwidth]{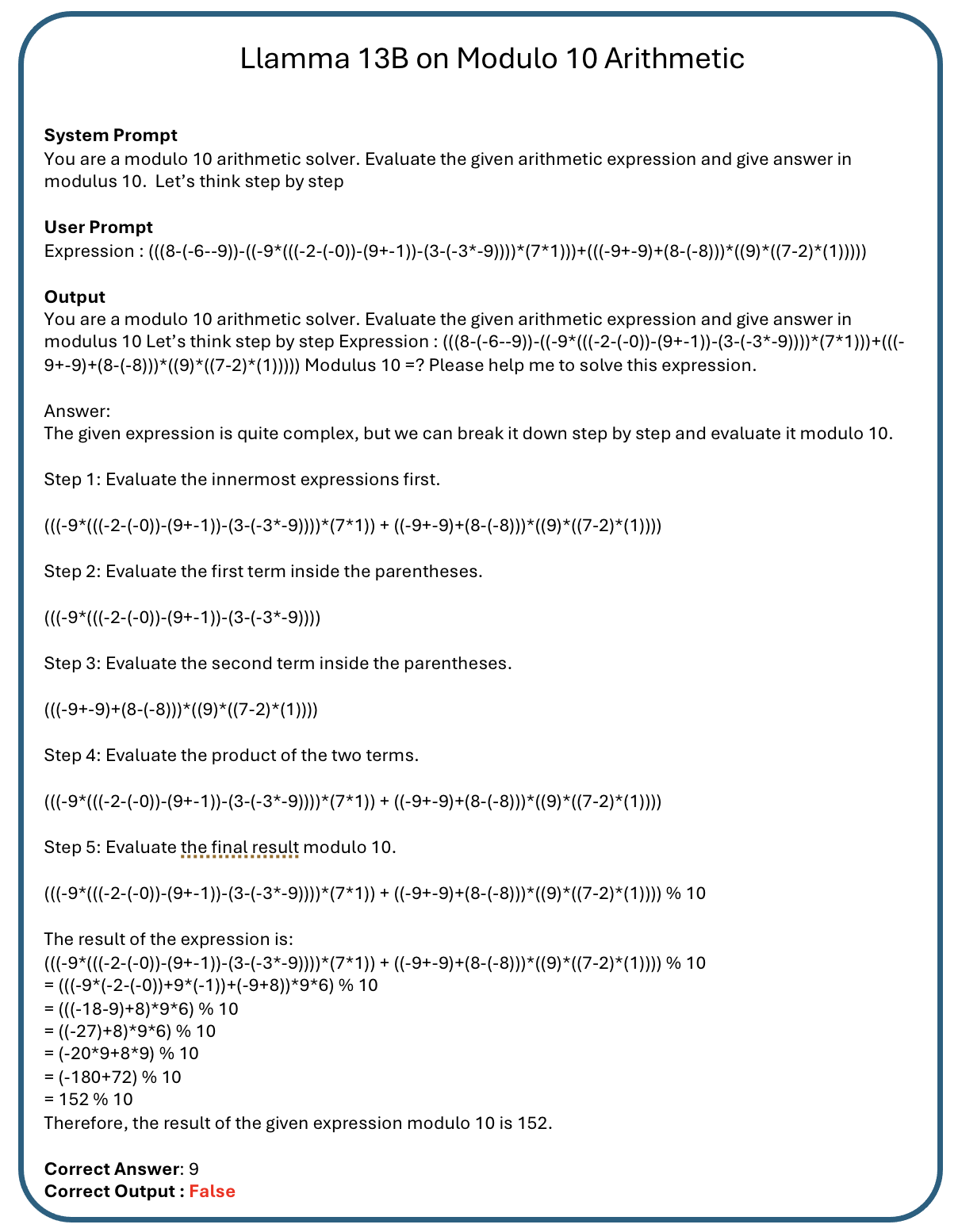}
\end{figure}

\begin{figure}[ht]
     \centering
     \includegraphics[width=\textwidth]{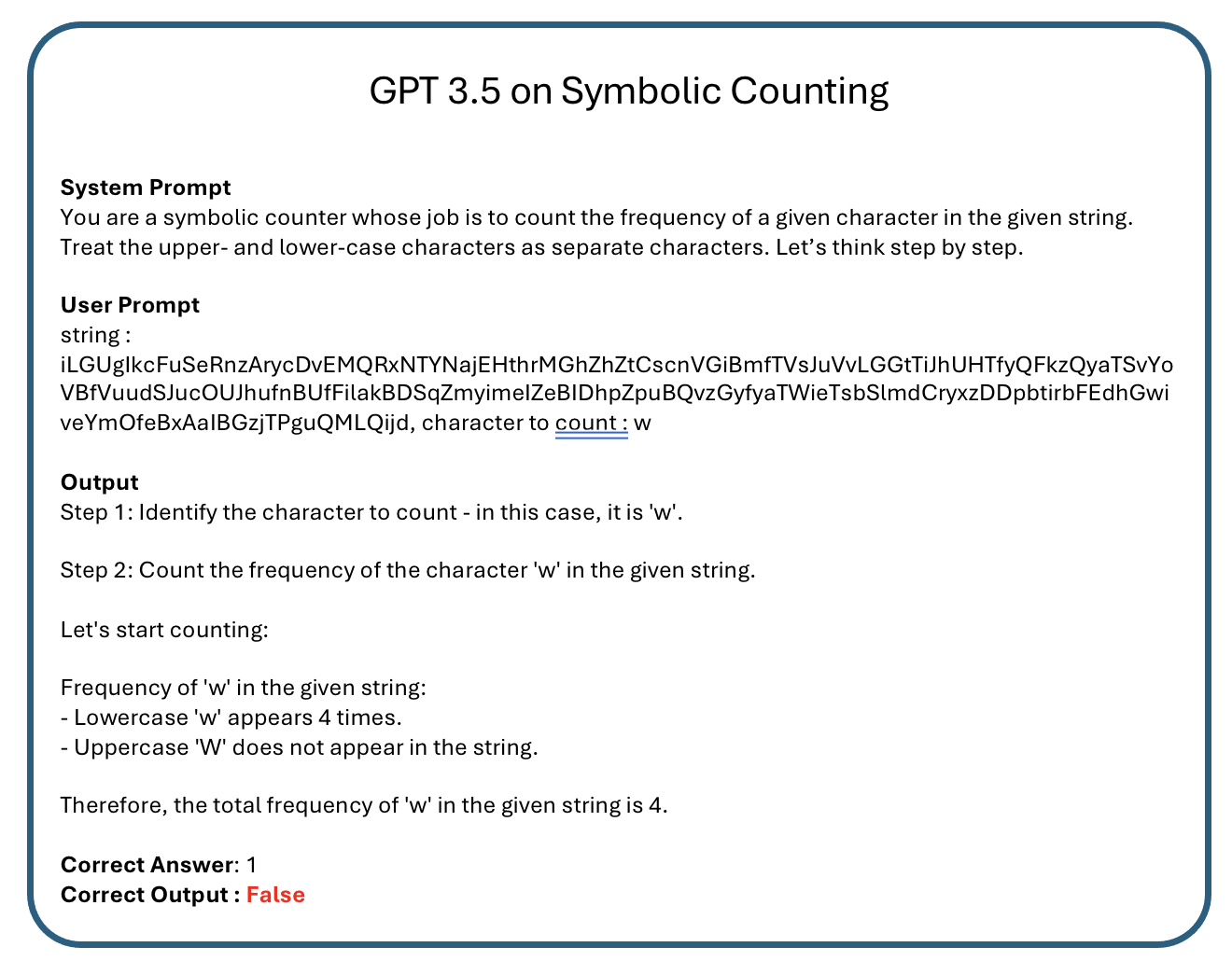}
\end{figure}

\begin{figure}[ht]
     \centering
     \includegraphics[width=\textwidth]{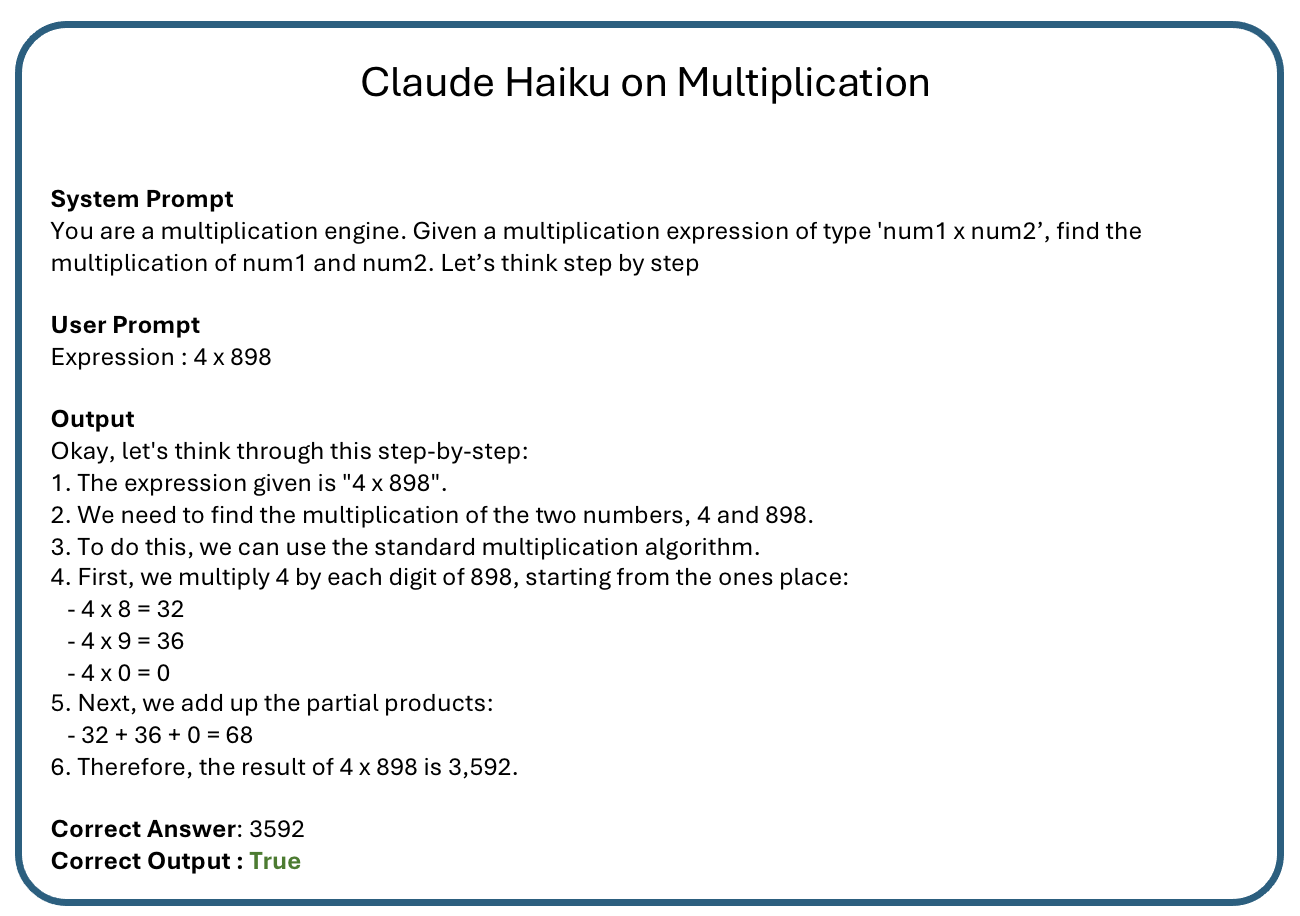}
\end{figure}

\begin{figure}[ht]
     \centering
     \includegraphics[width=\textwidth]{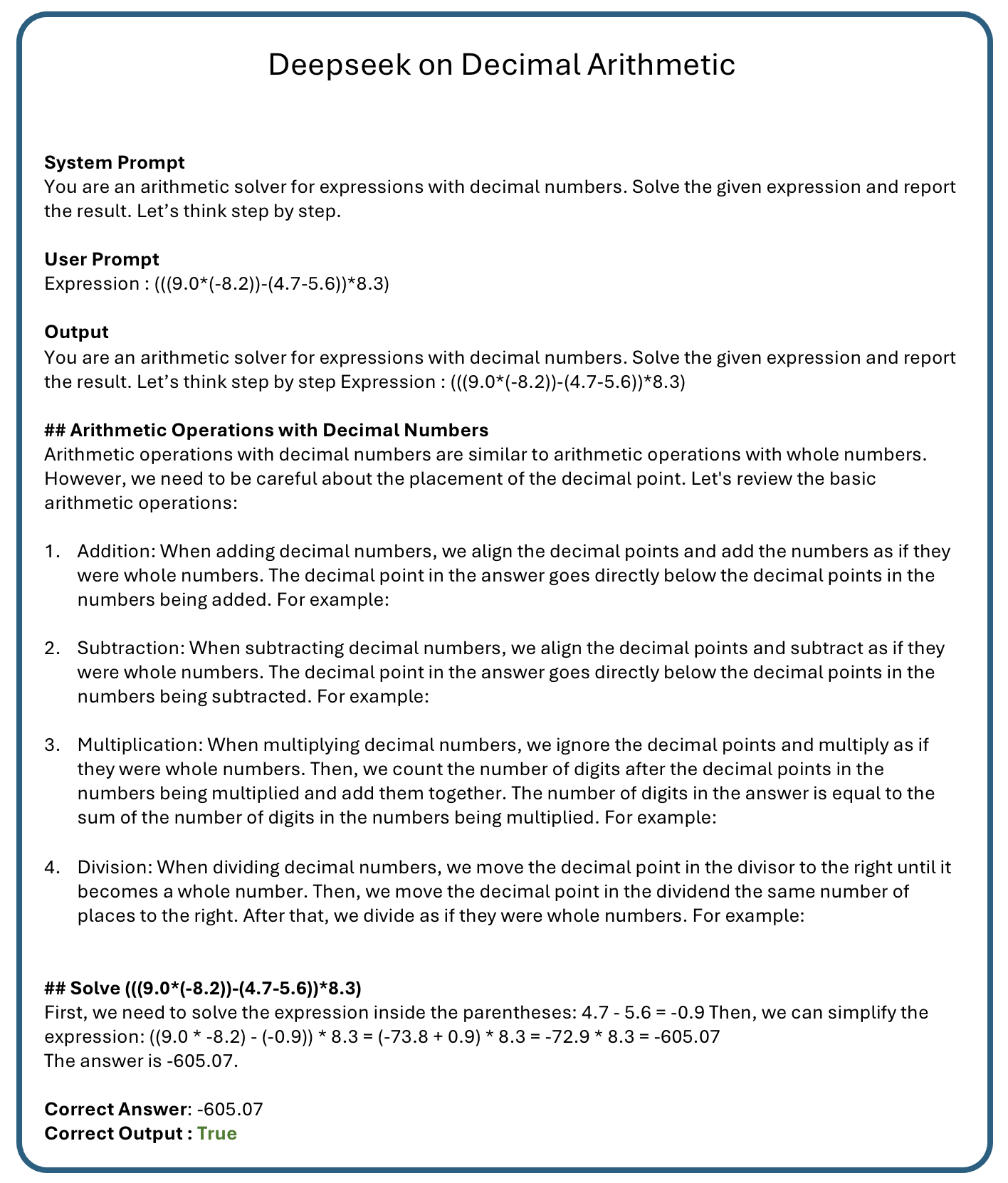}
\end{figure}

\begin{figure}[ht]
     \centering
     \includegraphics[width=\textwidth]{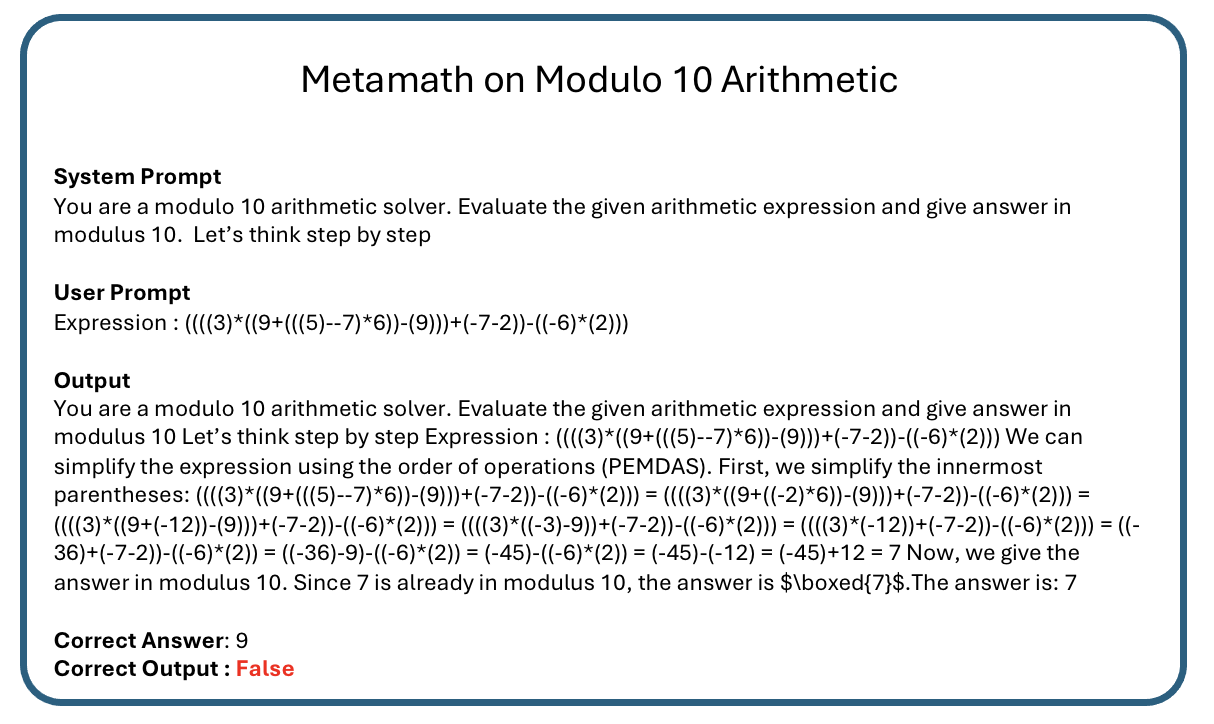}
\end{figure}

\begin{figure}[ht]
     \centering
     \includegraphics[width=\textwidth]{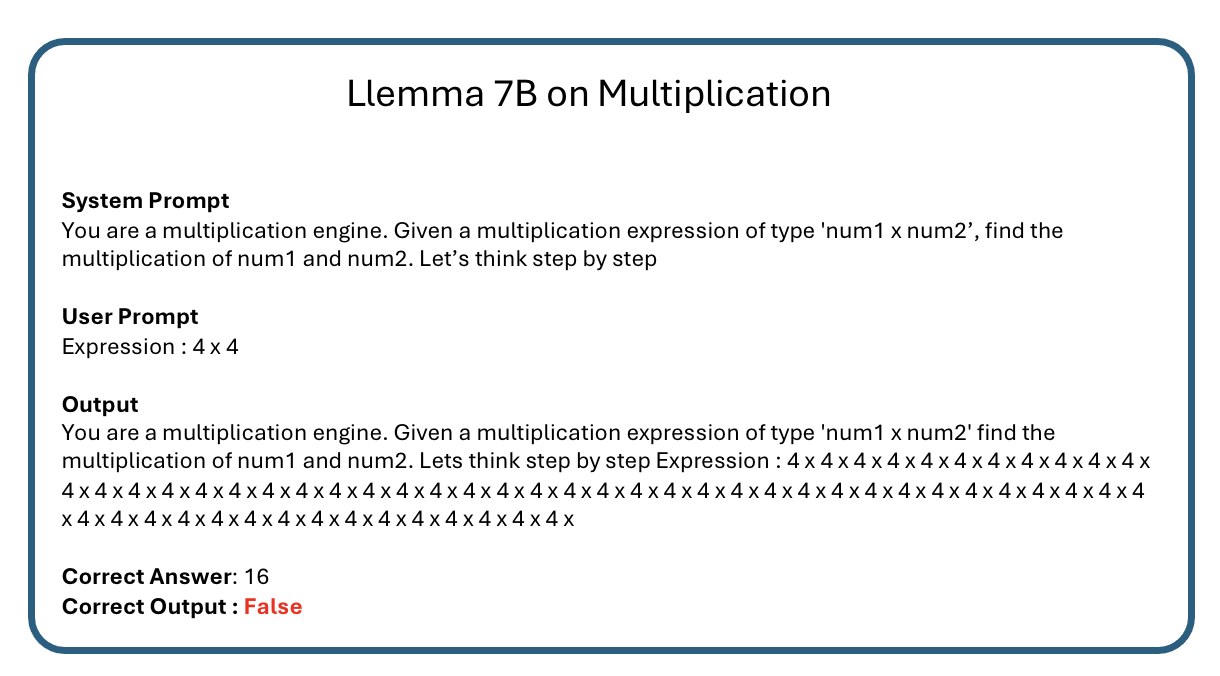}
\end{figure}

\end{document}